\documentclass[conference,a4paper]{IEEEtran}
\IEEEoverridecommandlockouts

\usepackage[utf8]{inputenc} 
\usepackage[T1]{fontenc}
\usepackage{url}             
\usepackage[sort, compress, space]{cite}            
\usepackage[cmex10]{amsmath}  
\usepackage{mleftright}       
\mleftright                   
\usepackage{graphicx}         
\usepackage{booktabs}         
\usepackage{algorithm, algorithmic}    
\usepackage[caption=false,font=footnotesize]{subfig}

\usepackage{amssymb, amsfonts} 
\usepackage{nicefrac}
\usepackage{mathtools}
\usepackage{bm, bbm}
\usepackage[scr=boondoxo,scrscaled=1.05]{mathalfa}
\usepackage{upgreek}

\usepackage{cleveref}

\usepackage{thm-restate}

\usepackage{dsfont}
\usepackage[shortlabels]{enumitem}
\usepackage{xcolor}
\usepackage{comment}
\usepackage{tikz}
\usepackage{tikzscale}
 
\usepackage{cuted}

\definecolor{orange}{rgb}{1,0.4,0.0}

\DeclarePairedDelimiterXPP{\KL}[2]{D_\textnormal{KL}}{(}{)}{}{%
#1\:\delimsize\|\:#2%
}
\DeclarePairedDelimiterXPP{\RD}[2]{D_{$\alpha$}}{(}{)}{}{%
#1\:\delimsize\|\:#2%
}

\newcommand{\innerproduct}[2]{\langle #1, #2 \rangle}

\DeclarePairedDelimiterXPP\Prob[1]{\mathbb{P}}{\lbrace}{\rbrace}{}{

#1}

\DeclarePairedDelimiterXPP{\lnorm}[2]{}{\lVert}{\rVert}{_{#2}}{#1}

\newcommand{\REG}{\ensuremath{\textnormal{REG}}}

\newcommand{\bE}{\ensuremath{\mathbb{E}}}

\newcommand{\bP}{\ensuremath{\mathbb{P}}}

\newcommand{\bR}{\ensuremath{\mathbb{R}}}

\newcommand{\cA}{\ensuremath{\mathcal{A}}}
\newcommand{\cB}{\ensuremath{\mathcal{B}}}

\newcommand{\cH}{\ensuremath{\mathcal{H}}}

\newcommand{\cN}{\ensuremath{\mathcal{N}}}
\newcommand{\cO}{\ensuremath{\mathcal{O}}}

\newcommand{\cX}{\ensuremath{\mathcal{X}}}

\newcommand{\mi}{\textup{I}}
\newcommand{\ent}{\textup{H}}
\newcommand{\kl}{\textup{D}_{\textnormal{KL}}}

\newtheorem{definition}{Definition}

\newtheorem{theorem}{Theorem}

\newtheorem{assumption}{Assumption}

\newcommand\xqed[1]{%
  \leavevmode\unskip\penalty9999 \hbox{}\nobreak\hfill
  \quad\hbox{#1}}

\newenvironment{proof}{\emph{Proof:}}{\xqed{$\blacksquare$}}

\begin{document}

\title{Thompson Sampling Regret Bounds for Contextual Bandits with sub-Gaussian rewards
\thanks{This work was partially supported by (i) the Wallenberg AI, Autonomous Systems and Software Program (WASP) funded by the Knut and Alice Wallenberg Foundation and (ii) the Swedish Research Council under contract 2019-03606.}}


\author{\IEEEauthorblockN{Amaury Gouverneur, Borja Rodríguez-Gálvez, Tobias J. Oechtering, and Mikael Skoglund} \IEEEauthorblockA{Division of Information Science and Engineering (ISE)}\IEEEauthorblockA{KTH Royal Institute of Technology}\IEEEauthorblockA{\texttt{\{amauryg,borjarg,oech,skoglund\}@kth.se}}}

\maketitle
\everypar{\looseness=-1}

\begin{abstract}
In this work, we study the performance of the Thompson Sampling algorithm for Contextual Bandit problems based on the framework introduced by \cite{neu2022lifting} and their concept of lifted information ratio. First, we prove a comprehensive bound on the Thompson Sampling expected cumulative regret that depends on the mutual information of the environment parameters and the history. Then, we introduce new bounds on the lifted information ratio that hold for sub-Gaussian rewards, thus generalizing the results from \cite{neu2022lifting} which analysis requires binary rewards. Finally, we provide explicit regret bounds for the special cases of unstructured bounded contextual bandits, structured bounded contextual bandits with Laplace likelihood, structured Bernoulli bandits, and bounded linear contextual bandits. 
\end{abstract}

\section{Introduction}
\label{sec:introduction}

Contextual bandits encompasses sequential decision-making problems where at each round an agent must choose an action that results in a reward. This action is chosen based on a context of the environment and a history of past contexts, rewards, and actions~\cite{langford2007epoch}.\footnote{
This setting is also known as bandit problems with covariates~\cite{sarkar1991one, woodroofe1979one}, associative reinforcement learning~\cite{barto1985pattern, gullapalli1990associative, kaelbling1994associative}, or associative bandit problems~\cite{strehl2006experience}.} Contextual bandits have become an important subset of sequential decision-making problems due to their multiple applications in healthcare, finance, recommender systems, or telecommunications (see~\cite{bouneffouf2020survey} for a survey on different applications).

There is an interest to study the theoretical limitations of algorithms for contextual bandits. This is often done considering their \emph{regret}, which is the difference in the collected rewards that an algorithm obtains compared to an oracle algorithm that chooses the optimal action at every round~\cite{abe2003reinforcement, dudik2011efficient, chu2011contextual, agarwal2012contextual, foster2020beyond, foster2021efficient,  zhang2022feel, neu2022lifting}.

A particularly successful approach is the \emph{Thomson Sampling (TS) algorithm}~\cite{thompson1933likelihood}, and was originally introduced for multi armed bandits, which are sequential decision-making problems without context. Despite its simplicity, this algorithm has been shown to work remarkably well for contextual bandits~\cite{scott2010modern, chapelle2011empirical}. This algorithm has been studied for multi armed bandits~\cite{russo2014learning, russo2016information, dong2018information} and in the more general context of Markov decision processes~\cite{gouverneur2022information}. A crucial quantity for the analysis of TS in the multi armed bandit setting is the \emph{information ratio} \cite{russo2014learning}, which trades off achieving low regret and gaining information about the optimal action. 

In \cite{neu2022lifting}, the authors extend this concept to the \emph{lifted information ratio} to fit the more challenging setting of contextual bandits, where the optimal action changes at every round based on the context. However, their main results are limited to contextual bandits with binary rewards. Albeit this is a common setting, as often rewards represent either a success or a failure~\cite{chapelle2011empirical}, it fails to capture more nuanced scenarios, like dynamic pricing where rewards represent revenue~\cite{mueller2019low}.

In this paper, we extend the results from \cite{neu2022lifting} to contextual bandits with sub-Gaussian rewards. These rewards include the common setup where the rewards are bounded, but are not necessarily binary~\cite{abe2003reinforcement, dudik2011efficient, chu2011contextual, agarwal2012contextual, foster2020beyond, foster2021efficient,  zhang2022feel}, or setups where the expected reward is linear but is corrupted by a sub-Gaussian noise~\cite{mueller2019low}.

More precisely, our contributions in this paper are:
\begin{itemize}
    \item A comprehensive bound on the TS regret that depends on the mutual information between the environment parameters and the history collected by the agent (\Cref{th:TS_regret_bound_mi}). Compared to \cite[Theorem 1]{neu2022lifting}, this bound highlights that, given an average lifted information ratio, the regret of TS does not depend on all the uncertainty of the problem, but only on the uncertainty that can be explained by the data collected from the TS algorithm. 
    \item An alternative proof of \cite[Theorem 2]{neu2022lifting} showing that, if the log-likelihood of the rewards satisfies certain regularity conditions, the TS regret is bounded by a measure of the complexity of the parameters' space in cases where this is not countable. The presented proof (\Cref{th:TS_regret_bound_net}) highlights that the rewards need not to be binary.
    \item Showing the lifted information ratio is bounded by the number of actions $|\cA|$ in unstructured settings (\Cref{lemma:finite_actions_information_ratio}) and by the dimension $d$ when the expected rewards are linear (\Cref{lemma:linear_expected_rewards}). These bounds extend \cite[Lemmata 1 and 2]{neu2022lifting} from the case where the rewards are binary to the more general setting where they are sub-Gaussian.
    \item \looseness=-1 Explicit regret bounds for particular settings as an application of the above results (\Cref{sec:applications}). Namely, bounds for (i) bounded unstructured contextual bandits that show that TS has a regret with the desired \cite{dudik2011efficient, beygelzimer2011contextual} rate of $O(\sqrt{|\cA|T \log |\cO|})$, (ii) bounded structured contextual bandits including those with Laplace likelihoods and Bernoulli bandits, and (iii) bounded linear bandits that show that the TS regret is competitive with \texttt{LinUCB}'s~\cite{chu2011contextual}.
\end{itemize}

\section{Preliminaries}
\label{sec:preliminaries}
\subsection{General Notation}
\label{subsec:notation}

Random variables $X$ are written in capital letters, their realizations $x$ in lowercase letters, their outcome space in calligraphic letters $\cX$, and its distribution is written as $\bP_X$. The density of a random variable $X$ with respect to a measure $\mu$ is written as $f_X \coloneqq \frac{d\bP_X}{d\mu}$. When two (or more) random variables $X, Y$ are considered, the conditional distribution of $Y$ given $X$ is written as $\bP_{Y|X}$ and the notation is abused to write their joint distribution as $\bP_X \bP_{Y|X}$.

\subsection{Problem Setting: Contextual Bandits}
\label{subsec:problem_setting}

A \emph{contextual bandit} is a sequential decision problem where, at each time step, or round $t \in [T]$, an agent interacts with an environment by observing a context $X_t \in \cX$ and by selecting an action $A_t \in \cA$ accordingly. Based on the context and the action taken, the environment produces a random reward $R_t \in \bR$. The data is collected in a history $H^{t+1} = H^t \cup H_{t+1}$, where $H_{t+1} = \{ A_t, X_t, R_t \}$. The procedure repeats until the end of the time horizon, or last round $t=T$. 

In the Bayesian setting, the environment is characterized by a parameter $\Theta \in \cO$ and a contextual bandit problem $\Phi$ is completely defined by a prior environment parameter $\bP_\Theta$, a context distribution $\bP_X$, and a fixed reward kernel $\kappa_{\textnormal{reward}}: \cB(\bR) \times (\cX, \cA, \cO) \to [0,1]$ such that $\bP_{R_t | X_t, A_t, \Theta} = \kappa_{\textnormal{reward}}\big(\cdot, (X_t, A_t, \Theta) \big)$. Thus, the reward may be written as $R_t = R(X_t, A_t, \Theta)$ for some (possibly random) function $R$.

The task in a Bayesian contextual bandit is to learn a policy $\varphi = \{ \varphi_t : \cX \times \cH^t \to \cA \}_{t=1}^T$ taking an action $A_t$ based on the context $X_t$ and on the past collected data $H^t$ that maximizes the \emph{expected cumulative reward} $R_\Phi(\varphi) \coloneqq \bE \big[\sum_{t=1}^T R(X_t, \varphi_t(X_t, H^t), \Theta) \big]$. 

\subsubsection{The Bayesian expected regret}
\label{subsubsec:bayesian_expected_regret}

The Bayesian expected regret of a contextual bandit problem measures the difference between the performance of a given policy and the optimal one, which is the policy that knows the true reward function and selects the actions yielding the highest expected reward.
For a given contextual bandit problem, we define the performance of the optimal policy as the \emph{optimal cumulative reward}.

\begin{definition}
    \label{def:optimal_cumulative_reward}
    The \emph{optimal cumulative reward} of a contextual bandit problem $\Phi$ is defined as
    \begin{equation*}
        R^\star_\Phi \coloneqq \sup_{\psi} \bE \bigg[ \sum_{t=1}^T R(X_t, \psi(X_t, \Theta), \Theta) \bigg],
    \end{equation*}
    where the supremum is taken over the decision rules $\psi: \cX \times \cO \to \cA$ such that the expectation above is defined.
\end{definition}

A policy that achieves the supremum of~\Cref{def:optimal_cumulative_reward} is denoted as $\psi^\star$ and the actions it generates are $A^\star_t \coloneqq \psi^\star(X_t, \Theta)$.

\begin{assumption}[Compact action set]
    \label{ass:compact_action_set}
    The set of actions $\cA$ is compact. Therefore, an optimal policy $\psi^\star$ always exists.
\end{assumption}

The difference between the expected cumulative reward of a policy $\varphi$ 
and the optimal cumulative reward is the \emph{Bayesian expected regret}.

\begin{definition}
    \label{def:bayesian_expected_regret}
    The \emph{Bayesian expected regret} of a policy $\varphi$ in a contextual bandit problem $\Phi$ is defined as
    \begin{equation*}
        \REG_\Phi(\varphi) \coloneqq R^\star_\Phi - R_\Phi(\varphi).
    \end{equation*}
\end{definition}

\subsubsection{The Thompson sampling algorithm}
\label{subsubsec:thompson_sampling_algorithm}

Thomson Sampling (TS) is an elegant algorithm to solve decision problems when the environment $\Theta$ is unknown. It works by randomly selecting actions according to their posterior probability of being optimal. More specifically, at each round $t \in [T]$, the agent samples a Bayes estimate $\hat{\Theta}_t$ of the environment parameters $\Theta$ based on the past collected data $H^t$ and selects the action given the optimal policy $\psi^\star$ for the estimated parameters and the observed context $X_t$, that is $\hat{A}_t = \psi^\star(X_t, \hat{\Theta}_t)$. The history collected by the TS algorithm up to round $t$ is denoted $\hat{H}^t$. The pseudocode for this procedure is given in~\Cref{alg:Thompson_Sampling}. Therefore, the Bayesian cumulative reward $R^\textnormal{TS}_\Phi$ of the TS algorithm is
\begin{equation*}
    R^\textnormal{TS}_\Phi \coloneqq \bE \bigg[ \sum_{t=1}^T R(X_t, \psi^\star(X_t, \hat{\Theta}_t), \Theta) \bigg],
\end{equation*}
where $\hat{\Theta}_t$ has the property that $\bP_{\hat{\Theta}|\hat{H}^t} = \bP_{\Theta|\hat{H}^t}$ a.s.. The Bayesian expected regret of the TS is denoted $\REG^\textnormal{TS}_\Phi$ and is usually referred to as the \emph{TS cumulative regret}.

\subsubsection{Notation specific to contextual bandits}
\label{subsecsub:problem_setting}

To aid the exposition, and since the $\sigma$-algebras of the history $\hat{H}^t$ and the context $X_t$ are often in the conditioning of the expectations and probabilities used in the analysis, similarly to~\cite{russo2016information, neu2022lifting}, we define the operators $\bE_t[\cdot] \coloneqq \bE[\cdot|\hat{H}^t, X_t]$ and $\bP_t[\cdot] \coloneqq \bP[\cdot|\hat{H}^t, X_t]$, whose outcomes are $\sigma(\cH^{t} \times \cX)$-measurable random variables and $\cH = \cA \times \cX \times \bR$. Similarly, we define $\mi_t(\Theta; R_t | \hat{A}_t) \coloneqq \bE_t[ \kl( \bP_{R_t | \hat{H}^t, X_t, \hat{A}_t, \Theta} \lVert \bP_{R_t | \hat{H}^t, X_t, \hat{A}_t} )]$ as the \emph{disintegrated} conditional mutual information between the parameter $\Theta$ and the reward $R_t$ given the action $\hat{A}_t$, \emph{given the history $\hat{H}^t$ and the context $X_t$}, see~\cite[Definition 1.1]{negrea2019information}, which is itself as well a $\sigma(\cH^t \times \cX)$-measurable random variable.

\begin{algorithm}[ht]
    \caption{Thompson Sampling algorithm}
    \label{alg:Thompson_Sampling}
    \begin{algorithmic}[1]
        \STATE {\bfseries Input:} environment parameters prior $\bP_{\Theta}$.
        \FOR{$t=1$ {\bfseries to} T}
            \STATE Observe the context $X_t \sim \bP_{X}$.
            \STATE Sample a parameter estimation $\smash{\hat{\Theta}_t \sim \bP_{\Theta|\hat{H}^t}}$.
            \STATE Take the action $\hat{A}_t = \psi^\star(X_t,\hat{\Theta}_t)$.
            \STATE Collect the reward $R_t = R(X_t,\hat{A}_t, \Theta)$.
            \STATE Update the history $\hat{H}^{t+1}=\{\hat{H}^t,\hat{A}_t,X_t,R_t\}$.
        \ENDFOR
    \end{algorithmic}
\end{algorithm}

\section{Main results}
\label{sec:main}
In this section, we present our main results to bound the TS cumulative regret for contextual bandits. In~\Cref{subsec:bounding_ts_cumulative_regret}, we first (\Cref{th:TS_regret_bound_mi}) prove a comprehensive bound on the TS cumulative regret that, rather than depending on the entropy of the environment's parameters as \cite[Theorem 1]{neu2022lifting}, it depends on their mutual information with the history. This highlights that, given an average lifted information ratio, the TS cumulative regret does not depend on the uncertainty of the parameters, but on the uncertainty of the parameters explained by the history.  Then (\Cref{th:TS_regret_bound_net}), we slightly relax the assumptions of~\cite[Theorem 2]{neu2022lifting} and digest this result with an alternative proof, which formalizes that the TS cumulative regret is bounded by the complexity of the environment's space. In~\Cref{subsec:Bounding the lifted information ratio}, we provide bounds on the lifted information ratio. First (\Cref{lemma:finite_actions_information_ratio}), without assuming any structure in the rewards, we show a bound that scales linearly with the number of actions. We then (\Cref{lemma:linear_expected_rewards}) consider the special case of linear contextual bandits and show that in that case we can obtain a bound that scales with the dimension of the problem. These results, in turn, generalize \cite[Lemmata 1 and 2]{neu2022lifting}, which are only valid for binary losses.

\subsection{Bounding the TS cumulative regret}
\label{subsec:bounding_ts_cumulative_regret}

In the contextual bandits setting, the concept of \emph{lifted information ratio} was introduced in \cite{neu2022lifting} as the random variable 
\begin{equation*}
    \Gamma_t \coloneqq \frac{\bE_t[ R^\star_t - R_t]^2}{\mi_t(\Theta; R_t|\hat{A}_t)},
\end{equation*}
where $R_t$ is the reward collected by the TS algorithm and $R^\star_t$ is the one collected playing optimally, i.e. $R(X_t, \psi^\star_t(X_t, \Theta), \Theta)$. This concept was inspired by the \emph{information ratio} from~\cite{russo2016information} in the non-contextual multi armed bandit problem setting and it is closely related to the \emph{decoupling coefficient} from~\cite{zhang2022feel}.

In the proof of~\cite[Theorem 1]{neu2022lifting}, it is shown that
\begin{equation}
    \REG_\Phi^\textnormal{TS} \leq \sqrt{ \bigg(\sum_{t=1}^T \bE[\Gamma_t] \bigg) \bigg(\sum_{t=1}^T \mi(\Theta; R_t | \hat{H}^t, X_t, \hat{A}_t) \bigg)}.
    \label{eq:intermediate_result_neu}
\end{equation}
This is employed to show a result bounding the TS cumulative regret for problems with a countable environment space $\Theta$. However, this intermediate step can also be leveraged to obtain a more general, and perhaps more revealing bound on the TS cumulative regret.

\begin{restatable}{theorem}{TS_regret_bound_mi}
\label{th:TS_regret_bound_mi}
Assume that the average of the lifted information ratios is bounded $\frac{1}{T} \sum_{t=1}^T \bE[\Gamma_t] \leq \Gamma$ for some $\Gamma > 0$. Then, the TS cumulative regret is bounded as
\begin{align}
\REG^\textnormal{TS}_{\Phi}
       &\leq  \sqrt{\Gamma T \mi(\Theta; \hat{H}^{T+1})} \nonumber \\ 
       &= \sqrt{\Gamma T \bE[\kl(\bP_{\Theta| \hat{H}^{T+1}} \lVert \bP_\Theta)]} . \nonumber
       \label{eq:ts_regret_bound_mi}
\end{align} 
\end{restatable}

\begin{proof}
   The proof follows by an initial application of the chain rule of the mutual information. Namely,
   \begin{align*}
       \smash{\mi(\Theta; \hat{H}^{T+1}) = \sum\nolimits_{t=1}^T \mi(\Theta; \hat{H}_{t+1} | \hat{H}^t).}
   \end{align*}
   Applying the chain rule once more to each term shows that
   \begin{align*}
       \mi(\Theta; \hat{H}_{t+1} | \hat{H}^t) = \mi(\Theta; X_t, \hat{A}_t | \hat{H}^t) +  \mi(\Theta; R_t | \hat{H}^t, X_t, \hat{A}_t).
   \end{align*}
   Finally, the non-negativity of the mutual information completes the proof as $\mi(\Theta; \hat{H}_{t+1} | \hat{H}^t) \geq  \mi(\Theta; R_t | \hat{H}^t, X_t, \hat{A}_t)$.
\end{proof}

\Cref{th:TS_regret_bound_mi} has~\cite[Theorem 1]{neu2022lifting} as a corollary by noting that for countable parameters' spaces $ \mi(\Theta; \hat{H}^{T+1}) \leq \ent(\Theta)$ and that if $\Gamma_t \leq \Gamma$ a.s. for all $t \in [T]$, then $\frac{1}{T} \sum_{t=1}^T \bE[\Gamma_t] \leq \Gamma$. This seemingly innocuous generalization gives us insights on the TS cumulative regret via the following two factors:
\begin{itemize}
    \item The bound on the average of lifted information ratios $\Gamma$. This measures the maximum information gain on the environment parameters on average through the rounds. This is different to the requirement that $\bE[\Gamma_t] \leq \Gamma'$ from \cite{neu2022lifting}, which penalizes equally rounds
    with large or little information gain. 
    This may be relevant in scenarios where the lifted information ratio can vary drastically among rounds.
    \item The mutual information between the parameters $\Theta$ and the history $\hat{H}^t$. 
    Contrary to the entropy $\ent(\Theta)$ featured in the bound \cite[Theorem 1]{neu2022lifting}, 
    which is a measure of the uncertainty of the parameters, the mutual information $\mi(\Theta; \hat{H}^t)$ measures the uncertainty of the parameters that is explained by the history of TS since 

     \begin{equation*}
        \mi(\Theta; \hat{H}^t) = \underbrace{\ent(\Theta)}_{\textnormal{Uncertainty of $\Theta$}} - \underbrace{\ent(\Theta|\hat{H}^t).}_{\substack{\textnormal{Uncertainty of $\Theta$} \\ \textnormal{not explained by $\hat{H}^t$}}}
    \end{equation*}
    Moreover, the mutual information is the relative entropy between the TS posterior on the parameters and the true parameters' prior, i.e. $\bE[\kl(\bP_{\Theta| \hat{H}^{T+1}} \lVert \bP_\Theta)]$, which measures how well is the TS posterior aligned with the true parameters' distribution in the last round. As for the TS algorithm we can sample from the posterior $\bP_{\Theta| \hat{H}^{T+1}}$, there are situations where the posterior is known analytically and thus this relative entropy can be numerically estimated at each round~\cite[Section 6]{russo2014learning}. 
\end{itemize}

In~\cite{neu2022lifting}, for binary rewards, i.e. $R: \cX \times \cA \times \cO \to \{ 0, 1 \}$, it is shown that regularity on the reward's log-likelihood is sufficient to guarantee a bound on the TS cumulative regret \emph{à la Lipschitz maximal inequality} \cite[Lemma 5.7]{van2014probability}. More precisely, if the parameters' space $\cO$ is a metric space $(\cO, \rho)$, they impose that the log-likelihood is Lipschitz continuous for all actions and all contexts. However, requiring the log-likelihood random variable to be a Lipschitz process is sufficient, as we will show shortly. 

\begin{assumption}[Lipschitz log-likelihood]
    \label{ass:regularity_condition_likelihood}
    There is a random variable $C > 0$ that can depend only on $R_t, X_t$, and $\hat{A}_t$ such that $| \log f_{R_t|X_t, \hat{A}_t, \Theta= \theta}(R_t) - \log f_{R_t|X_t, \hat{A}_t, \Theta=\theta'}(R_t)| \leq C \rho(\theta, \theta')$ a.s. for all $\theta, \theta' \in \cO$.
\end{assumption}

With this regularity condition, the TS cumulative regret can be bounded from above by the ``complexity" of the parameter's space $\cO$, measured by the $\epsilon$-covering number of the space.

\begin{definition}
    \label{def:epsilon_net}
    A set $\cN$ is an $\epsilon$-net for $(\cO,\rho)$ if for every $\theta \in \cO$, there exists a \emph{projection map} $\pi(\theta) \in \cN$ such that $\rho(\theta, \pi(\theta)) \leq \epsilon$. The smallest cardinality of an $\epsilon$-net for $(\cO, \rho)$ is called \emph{the $\epsilon$-covering number}
    \begin{equation*}
        |\cN(\cO, \rho, \epsilon)| \coloneqq \inf \{ | \cN | : \cN \textnormal{ is an } \epsilon\textnormal{-net for } (\cO, \rho) \}.
    \end{equation*}
\end{definition}

In~\cite{neu2022lifting}, they prove their result manipulating the densities and employing the \emph{Bayesian telescoping} technique to write the so called ``Bayesian marginal distribution" as the product of ``posterior predictive distributions"~\cite{grunwald2012safe}. Observing their proof, it seems that their result did not require the rewards to be binary to hold. Below, using the properties of mutual information and standard arguments to bound Lipschitz processes~\cite[Section 5.2]{van2014probability} we provide an alternative proof for this result where the weaker regularity condition and the unnecessary requirement of binary rewards is apparent.

\begin{theorem}
    \label{th:TS_regret_bound_net}
    Assume that the parameters' space is a metric space $(\cO, \rho)$ and let $|\cN(\cO, \rho, \varepsilon)|$ be the $\epsilon$-covering number of this space for any $\varepsilon > 0$. Assume as well that the log-likelihood is a Lipschitz process according to~\Cref{ass:regularity_condition_likelihood} and that the average of the lifted information ratios is bounded $\frac{1}{T} \sum_{t=1}^T \bE[\Gamma_t] \leq \Gamma$ for some $\Gamma > 0$. Then, the TS cumulative regret is bounded as
    \begin{equation*}
        \REG_\Phi^{\textnormal{TS}} \leq \sqrt{\Gamma T \min_{\varepsilon > 0 }\big \{ \varepsilon \bE[C] T + \log | \cN(\cO, \rho, \varepsilon) | \big \} }.
    \end{equation*}
\end{theorem}
\begin{proof}
    The proof follows considering~\eqref{eq:intermediate_result_neu} again. The mutual information terms can be written as
    \begin{align}
        \label{eq:conditional_mi_wrt_likelihoods}
        \mi(\Theta; R_t | \hat{H}^t, X_t, \hat{A}_t) = \bE \bigg[ \log \frac{f_{R_t | \hat{H}^t, X_t, \hat{A}_t, \Theta}(R_t)}{f_{R_t | \hat{H}^t, X_t, \hat{A}_t}(R_t)} \bigg].
    \end{align}

    Consider now an $\varepsilon$-net of $\cO$ with minimal cardinality $|\cN(\cO, \rho, \epsilon)|$, where $\pi$ is its projecting map. Then, the mutual information in~\eqref{eq:conditional_mi_wrt_likelihoods} can equivalently be written as
    \begin{align*}
        \bE \bigg[ \int_\cO  f_{\Theta|R_t, \hat{H}^t, X_t, \hat{A}_t}(\theta) &\bigg( \log \frac{f_{R_t | X_t, \hat{A}_t, \Theta=\theta}(R_t)}{f_{R_t | X_t, \hat{A}_t, \Theta = \pi(\theta)}(R_t)} \\
        &+ \log \frac{f_{R_t | \hat{H}^t, X_t, \hat{A}_t, \Theta=\pi(\theta)}(R_t)}{f_{R_t | \hat{H}^t, X_t, \hat{A}_t}(R_t)} \bigg) d\theta \bigg],
    \end{align*}
    since $f_{R_t | \hat{H}^t, X_t, \hat{A}_t, \Theta} = f_{R_t | X_t, \hat{A}_t, \Theta}$ a.s. by the conditional Markov chain $R_t - \hat{A}_t - \hat{H} \  | \ \Theta, X_t$. The regularity condition in~\Cref{ass:regularity_condition_likelihood} ensures that the first term is bounded by $\varepsilon \bE[C]$. Then, defining the random variable $\Theta_\pi \coloneqq \pi(\Theta)$, we note that the second term is equal to $\mi(\Theta_\pi; R_t| \hat{H}^t, X_t, \hat{A}_t)$.

    Summing the $T$ terms from the regularity condition results in $\varepsilon \bE[C] T$ and, similarly to the proof of~\Cref{th:TS_regret_bound_mi}, summing the $T$ mutual information $\mi(\Theta_\pi; R_t | \hat{H}^t, X_t, \hat{A}_t)$ terms results in the upper bound
    \begin{equation*}
        \smash{\sum\nolimits_{t=1}^T \mi(\Theta_\pi; R_t | \hat{H}^t, X_t, \hat{A}_t) \leq \mi(\Theta_\pi; \hat{H}^{T+1}) \leq \ent(\Theta_\pi).}
    \end{equation*}
    Finally, bounding the entropy by the cardinalitiy of the net $\ent(\Theta_\pi) \leq \log |\cN(\cO, \rho, \varepsilon)|$ completes the proof.
    
\end{proof}

\subsection{Bounding the lifted information ratio}
\label{subsec:Bounding the lifted information ratio}

The next lemma provides a bound on the lifted information ratio that holds for settings with a finite number of actions and sub-Gaussian rewards. This result generalizes~\cite[Lemma 1]{neu2022lifting} as their proof technique requires the rewards to be binary. Under this specific case, we recover their result with a smaller constant as binary random variables are $1/4$-sub-Gaussian.\footnote{Random variables in $[0,L]$ are $\frac{L^2}{4}$-sub-Gaussian\cite[Theorem 1]{hoeffding1994probability}.}
\begin{restatable}{lemma}{FiniteActions}
\label{lemma:finite_actions_information_ratio}
Assume the number of actions $|\cA|$ is finite. If for all $t \in [T]$, $h^t\in \cH^t$, and $x\in\cX$, the random rewards $R_t$ are $\sigma^2$-sub-Gaussian under $\bP_{R_t|\hat{H}^t=h^t, X_t=x}$, then $\Gamma_t \leq 2\sigma^2 |\cA|$.
\end{restatable}

\begin{proof}
The proof adapts \cite[Proof of Proposition 3]{russo2014learning} to contextual bandits. The adaptation considers sub-Gaussian rewards using the Donsker--Varadhan inequality~\cite[Theorem 5.2.1]{gray2011entropy} as suggested in \cite[Appedix D]{russo2014learning}. This adaptation completely differs from the one in~\cite{neu2022lifting}, which is based on convex analysis of the relative entropy of distributions with binary supports. The full proof is in Appendix~\ref{app:proofs_of_lemmata}.
\end{proof}

Next, we consider cases of linear expected rewards. This setting is an extension of the stochastic linear bandit problem studied in \cite[Section 6.5]{russo2016information} to contextual bandit problems. The following lemma provides a bound on the lifted information ratio for problems in this setting with sub-Gaussian rewards, thus generalizing~\cite[Lemma 2]{neu2022lifting} which only considers binary random rewards. It useful in cases where the dimension is smaller than the number of actions $d < |\cA|$.

\begin{restatable}{lemma}{LinearExpectedRewards}
\label{lemma:linear_expected_rewards}
Assume the number of actions $|\cA|$ is finite, the expectation of the rewards is $\bE[R(x,a,\theta)]=\innerproduct{\theta}{m(x,a)}$ for some feature map $m: \cX \times \cA \to \bR^d$, and that $\cO \subseteq \bR^d$. If for all $t \in [T]$, $h^t\in \cH^t$, and $x\in\cX$, the random rewards $R_t$ are $\sigma^2$-sub-Gaussian under $\bP_{R_t|\hat{H}^t=h^t, X_t=x}$, then $\Gamma_t \leq 2\sigma^2 d$.

\end{restatable}
\begin{proof}
    The proof adapts \cite[Proof of Proposition 5]{russo2014learning} to contextual bandits similarly to \cite[Proof of Lemma 2]{neu2022lifting}. The key difference with the latter is that instead of binary rewards~\cite{neu2022lifting}, this considers sub-Gaussian ones using again the Donsker--Varadhan inequality~\cite[Theorem 5.2.1]{gray2011entropy} similarly to the proof of \Cref{lemma:finite_actions_information_ratio}.
    The full proof is in Appendix~\ref{app:proofs_of_lemmata}.
\end{proof}

\section{Applications}
\label{sec:applications}

\subsection{Unstructured bounded contextual bandits}
\label{subsec:unstructured_bounded_contextual_bandits}

The problem of contextual bandits with bounded rewards $R: \cX \times \cA \times \cO \to [0,1]$  and a finite number of actions $|\cA|$ and of parameters $|\cO|$ is well studied. In \cite{dudik2011efficient} and \cite{beygelzimer2011contextual}, respectively, the authors showed that the algorithms \texttt{Policy Elimination} and \texttt{Exp4.P} have a regret upper bound in $O\big(\sqrt{|\cA| T \log ( T |\cO| / \delta)} \big)$ and  in $O\big(\sqrt{|\cA| T \log ( |\cO| / \delta)} \big)$  with probability at least $1-\delta$. Then, it was shown that there exist some contextual bandit algorithm with a regret upper bound in $O(\sqrt{|\cA| T \log |\cO|})$~\cite{foster2020beyond} and that,  for all algorithms, there is a parameters' space $\cO'$ with cardinality smaller than $|\cO|$ such that the regret lower bounded is in $\Omega(\sqrt{|\cA| T \log |\cO| / \log |\cA|})$~\cite{agarwal2012contextual}. This sparked the interest to study how the TS or related algorithms' regret compared to these bounds. In \cite[Section 5.1]{zhang2022feel}, it was shown that the \texttt{Feel-Good TS} regret has a rate in $O(\sqrt{|\cA| T \log |\cO|})$ and recently, in \cite[Theorem 3]{neu2022lifting}, it was shown that if the reward is binary, the TS also has a rate in $O(\sqrt{|\cA| T \log |\cO|})$. Here, as a corollary of \Cref{th:TS_regret_bound_mi} and \Cref{lemma:finite_actions_information_ratio}, we close the gap on the regret of the TS algorithm showing that it is in  $O(\sqrt{|\cA| T \log |\cO|})$ for sub-Gaussian rewards, and thus for bounded ones.

\begin{restatable}{corollary}{BoundedRewardsBound}
\label{cor:bounded_rewards_bound}
Assume that the rewards are bounded in $[0,L]$. Then, for any contextual bandit problem $\Phi$, the TS cumulative regret after $T$ rounds is bounded as
\begin{align*}
\REG^\textnormal{TS}_{\Phi}
     \leq  \sqrt{\frac{L^2 |\cA| T \ent(\Theta)}{2} } .
\end{align*} 
\end{restatable}
Note that the above result also holds for $\sigma^2$-sub-Gaussian rewards by replacing $L^2/2$ by $2 \sigma^2$.  

\subsection{Structured bounded contextual bandits}
\label{subsec:structured_bounded_contextual_bandits}

\subsubsection{Bandits with Laplace likelihoods}

We introduce the setting of contextual bandits with Laplace likelihoods. 
In this setting, we model the rewards' random variable with a Laplace distribution. More precisely, this setting considers rewards with a likelihood proportional to $ \exp \Big(-\frac{| r-f_{\theta}(x,a)|}{\beta}\Big)$ for some $\beta > 0$. In addition, this setting assumes that the random variable $f_{\theta}(X,A)$ is a Lipschitz process with respect to $\theta$ with random variable $C \coloneqq C(X,A)$. This ensures \Cref{ass:regularity_condition_likelihood} with random variable $\frac{C}{\beta}$ as by the triangle inequality
\begin{equation*}
    | r - f_\theta(x,a) | - | r - f_{\theta'}(x,a) | \leq | f_\theta(x,a) - f_{\theta'}(x,a) |.
\end{equation*}
\Cref{th:TS_regret_bound_net} and   \Cref{lemma:finite_actions_information_ratio} yield the following corollary, where we further use the bound on the $\varepsilon$-covering number $| \cN(\cO, \rho, \epsilon) | \leq \big(\frac{3S}{\varepsilon} \big)^d$~\cite[Lemma 5.13]{van2014probability} and we let $\varepsilon = \frac{d \beta }{\bE[C] T}$.

\begin{restatable}{corollary}{LaplacianBandits}
\label{cor:energy_based_rewards}
    Assume that $\cO \subset \mathbb{R}^d$ with $\textnormal{diam}(\cO) \leq S$. Consider a contextual bandit problem $\Phi$ with Laplace likelihood and rewards bounded in $[0,L]$. Then, the TS cumulative regret after $T$ rounds is bounded as
    \begin{align*}
        \REG^\textnormal{TS}_{\Phi}
     \leq  
     \sqrt{\frac{L^2 |\cA| T d}{2}  \bigg( 1 + \log \bigg(\frac{3 S \bE[C] T}{d \beta} \bigg)\bigg) } .
    \end{align*}
\end{restatable}

In particular, for linear functions $f_\theta(x,a) = \innerproduct{\theta}{m(x,a)}$ with a bounded feature map, i.e. $\lVert m(x,a) \rVert \leq B$ for all $x \in \cX$ and all $a \in \cA$, then $C \leq B$ a.s..

\subsubsection{Bernoulli bandits with structure}

A common setting is that of Bernoulli contextual bandits, where the random rewards $R_t$ are binary and Bernoulli distributed~\cite{scott2010modern, chapelle2011empirical}. This is an attractive setting as binary rewards are usually modeled to measure success in e-commerce. In this setting, usually $R_t \sim \textnormal{Ber}\big( g \circ f_\Theta(X_t, \hat{A}_t) \big)$, where $g$ is a \emph{binomial link function} and $f$ is a linear function $f_\theta(x,a) = \innerproduct{\theta}{m(x,a)}$ for some feature map $m$. When the link function is the logistic function $g(z) = \sigma(z) \coloneqq (1+e^{-z})^{-1}$, $f$ is $C$-Lipschitz (e.g., when it is a linear function with a bounded feature map),  and the parameters' space is bounded $\lVert \theta \rVert \leq S$ for all $\theta \in \cO$, \cite{neu2022lifting} showed that the TS cumulative regret rate is in $O\big(\sqrt{|\cA|T d\log(SCT)} \big)$. This result is founded in their Theorem 2 and Lemma 1, and the fact that $\log \sigma$ is a $1$-Lipschitz function. We note that this is also true for other link functions such as the generalized logistic function $\sigma_\alpha(z) \coloneqq (1+e^{-z})^{-\alpha}$, whose $\log$ is $\alpha$-Lipschitz for all $\alpha > 0$, or the algebraic logistic function $\sigma_{\textnormal{alg}}(z) \coloneqq \frac{1}{2}(1 + \frac{z}{\sqrt{1+z^2}})$, whose $\log$ is $2$-Lipschitz. Moreover, we also note that with an appropriate choice of $\varepsilon$ as in \Cref{cor:energy_based_rewards}, these results improve their rate to $O\big(\sqrt{|\cA|T d\log(SCT/d)} \big)$.

\subsection{Bounded linear contextual bandits}
\label{subsec:linear_contextual_bandits}

In this section, we focus on the setting of contextual bandits with linear expected rewards. This setting has been introduced by~\cite{abe2003reinforcement} and further studied in \cite{chu2011contextual}. In this setting, the rewards are bounded in $[0,1]$ and their expectation is linear $\bE[R(x,a,\theta)] = \innerproduct{\theta}{m(x,a)}$ with a bounded feature map $m: \cX \times \cA \to [0,1]$ and parameters' space $\textnormal{diam}(\cO) = 1$.

In this setting,  \cite{chu2011contextual} showed that \texttt{LinUCB} has a regret bound in $O\big(\sqrt{dT \log^3( |\cA| T \log( T) / \delta) }\Big)$ with probability no smaller than $1- \delta$. The following corollary shows that if one is able to work with a discretized version $\cO_\varepsilon$ of $\cO$ with precision $\varepsilon$, i.e. $\cO_\varepsilon$ is an $\varepsilon$-net of $\cO$, then TS has a regret bound in $O\Big(\sqrt{d^2 T \log \big(\frac{3}{\varepsilon}}\big) \Big) $, which also follows from the bound on the $\varepsilon$-covering number $|\cN(\cO, \lVert \cdot \rVert, \varepsilon)| \leq \big( \frac{3}{\varepsilon} \big)^d$ \cite[Lemma 5.13]{van2014probability}. This bound is especially effective when the dimension $d$ is small or the number of actions $|\cA|$ is large. More precisely, it is tighter than \cite{chu2011contextual}'s bound when $d \log(1/\varepsilon) < \log^3 (|\cA| T \log T)$.

\begin{restatable}{corollary}{ContextualLinearBandits}
\label{cor:ContextualLinearBandits}
     Assume that $\cO = \{\theta_1, \ldots, \theta_{|\cO|}\}$ where $\theta \in \bR^d$. Consider a contextual bandit problem $\Phi$ with a finite number of actions $|\cA|$, rewards bounded in $[0,L]$ and such that the expectation of the rewards is $\bE[R(x,a,\theta)]=\innerproduct{\theta}{m(x,a)}$ for some feature map $m: \cX \times \cA \to \bR^d$. Then the TS cumulative regret after $T$ rounds is bounded as
    \begin{align*}
        \REG^\textnormal{TS}_{\Phi} \leq \sqrt{\frac{L^2  d T \log(|\cO|)}{2}}
    \end{align*}
\end{restatable}
\begin{proof}
It follows from \Cref{th:TS_regret_bound_mi} and \Cref{lemma:linear_expected_rewards}.
\end{proof}

\section{Conclusion}
\label{sec:conclusion}
In this paper, we showed in \Cref{th:TS_regret_bound_mi} that the TS cumulative regret for contextual bandit problems is bounded from above by the mutual information between the environment parameters and the history. Compared to~\cite[Theorem 1]{neu2022lifting},  this highlights that, given an average lifted information ratio, the regret of TS does not depend on all the uncertainty of the environment parameters, but only on the uncertainty that can be explained by the history collected by the algorithm. In \Cref{th:TS_regret_bound_net}, we provided an alternative proof to \cite[Theorem 2]{neu2022lifting} showing that the TS regret is bounded by the "complexity" of the parameters' space, where we highlighted that this result holds without the requirement of the rewards being binary.

In Lemmata~\ref{lemma:finite_actions_information_ratio} and~\ref{lemma:linear_expected_rewards}, we  provided bounds on the lifted information ratio that hold for contextual bandit problems with sub-Gaussian rewards. This includes the standard setting where the rewards are bounded~\cite{abe2003reinforcement, dudik2011efficient, chu2011contextual, agarwal2012contextual, foster2020beyond, foster2021efficient,  zhang2022feel}, and setups where the expected reward is linear but is corrupted by a sub-Gaussian noise~\cite{mueller2019low}, thus extending the results from \cite{neu2022lifting} that worked only with binary rewards.  When no structure of the problem is assumed, the lifted information ratio bound scales with the number of actions $|\cA|$ (\Cref{lemma:finite_actions_information_ratio}), and for problems with linear expected rewards, the bound scales with the dimension $d$ of the parameters' space $\cO$ (\Cref{lemma:linear_expected_rewards}).

Finally, we applied our results to some particular settings such as: bounded unstructured contextual bandits, for which TS has a regret with rate of $O(\sqrt{|\cA|T \log |\cO|})$; bounded structured contextual bandits including those with Laplace likelihoods and Bernoulli bandits; and lastly, bounded linear bandits underlining that TS has a regret bound competing with \texttt{LinUCB} \cite{chu2011contextual}.




\IEEEtriggeratref{17}
\bibliographystyle{IEEEtran}
\bibliography{references}

\newpage 
\appendices

\section{Proofs of lemmata}
\label{app:proofs_of_lemmata}

\FiniteActions*

\begin{proof}
    The proof follows the same methodology as~\cite[Proof of Proposition 3]{russo2016information}, taking care of the presence of contexts in the analysis. 
For the sake of brevity, we introduce the following notation $R'_t(a) \coloneqq R(X_t,a,\Theta)$ and recall the previously defined notations $A^\star_t \coloneqq \psi^\star(X_t,\Theta)$ and $\hat{A}_t \coloneqq \psi^\star(X_t,\hat{\Theta}_t)$. Then at each round $t \in [T]$, one can write the expected regret conditioned on $\hat{H}^t,X_t$ as 
\begin{align*}
    \bE_t[ R^\star_t - R_t] =\sum_{a\in\cA}&\bP_{t}[A^\star_t = a]\bE_t[R'_t(a)|A^\star_t=a]\\
    &-\sum_{a\in\cA}\bP_{t}[\hat{A}_t= a]\bE_t[R'_t(a)|\hat{A}_t=a] \ \textnormal{a.s.}.
\end{align*}
By definition of the TS algorithm $\bP_{t}[A^\star_t = a]= \bP_{t}[\hat{A}_t= a]$ a.s.. Observing as well that conditioned on $\hat{H}^t$ and $X_t$, the reward $R'_t(a)$ is independent of the TS action $\hat{A}_t$, the 
conditional expected regret can be a.s. rewritten as
\begin{align}
    \label{eq:TS_reg_time_t}
    \sum_{a\in\cA}\bP_{t}[A^\star_t = a]\big(\bE_t[R'_t(a)|A^\star_t=a]-\bE_t[R'_t(a)]\big). 
\end{align}
As the rewards are $\sigma^2$-sub-Gaussian, the difference of expectations in this last rewriting can be upper bounded using the Donsker-Varadhan inequality~\cite[Theorem~5.2.1]{gray2011entropy} as in~\cite[Lemma 3]{russo2014learning}. It then comes that~\eqref{eq:TS_reg_time_t} can be a.s. upper bounded by
\begin{align}
    \sum_{a\in\cA} \underbrace{\bP_{t}[A^\star_t = a]\sqrt{2 \sigma^2 \KL{\bP_{R'_t(a)|\hat{H}^t,X_t,A^\star_t = a}}{\bP_{R'_t(a)|\hat{H}^t,X_t}}}}_{\coloneqq v_a}.\label{eq:TS_reg_time_t_exp_dsk}
\end{align}
Using the Cauchy-Schwartz inequality, i.e. $$\sum_{a\in \cA} u_a v_a \leq \sqrt{\sum_{a\in \cA} u_a^2 \sum_{a\in \cA} v_a^2},$$ with $u_a = 1$ for all $a \in \cA$ and $v_a$ defined as above it follows that ~\eqref{eq:TS_reg_time_t_exp_dsk} is a.s. upper bounded by 
\begin{align*}
    &\sqrt{2 \sigma^2|\cA|\sum_{a\in\cA}\bP_{t}[A^\star_t = a]^2 } \\
    & \qquad \qquad \cdot \sqrt{\KL{\bP_{R'_t(a)|\hat{H}^t,X_t,A^\star_t = a}}{\bP_{R'_t(a)|\hat{H}^t,X_t}}}.
\end{align*}
Adding the non-negative extra terms $ 2 \sigma^2 |\cA| \sum_{a\in \cA}  \bP_{t}[A^\star_t = a] \sum_{b\in \cA \setminus a}\bP_{t}[A^\star = b ] \KL{\bP_{R'_t(b)|\hat{H}^t,X_t,A^\star_t = a}}{\bP_{R'_t(b)|\hat{H}^t,X_t}}$ in the square root gives
\begin{align*}
    \bE_t[ R^\star_t - R_t] \leq \sqrt{ 2 \sigma^2 |\cA| \mi_t(A^\star_t;R_t|\hat{A}_t)} \textnormal{ a.s.},
\end{align*}
using that $\mi_t(A^\star_t;R_t|\hat{A}_t) = \sum_{a,b\in \cA} \bP_{t}[A^\star_t = a] \bP_{t}[A^\star_t = b]\KL{\bP_{R'_t(b)|\hat{H}^t,X_t,A^\star_t = a}}{\bP_{R'_t(b)|\hat{H}^t,X_t}}$ a.s.. Then, as the Markov chain $A^\star_t - \Theta - R_t \ | \ \hat{H}^t, X_t, \hat{A}_t$ holds, by the data processing inequality $\mi_t(A^\star_t;R_t|\hat{A}_t) \leq \mi_t(\Theta;R_t|\hat{A}_t)$ a.s.. Squaring and reordering the terms yields the desired result.
\end{proof}

\LinearExpectedRewards*

\begin{proof}
    This proof follows the techniques from~\cite[Proof of Proposition 5]{russo2016information} taking care of the presence of contexts similarly to~\cite[Proof of Lemma 2]{neu2022lifting}. The difference with the latter is that instead of using Pinsker's inequality after noting that the expected value of a Bernoulli random variable is its probability of success, restriting the analysis to binary rewards, it uses the Donsker--Varadhan inequality~\cite[Theorem 5.2.1]{gray2011entropy} as in the proof of \Cref{lemma:finite_actions_information_ratio} to allow sub-Gaussian rewards in the analysis.
    
    Let $\cA = \{a_1,\ldots, a_{|\cA|}\}$ without loss of generality and for any round $t \in [T]$, conditioned on the history $\hat{H}^t$ and the context $X_t$, we define a random matrix $M\in \mathbb{R}^{|\cA|\times |\cA|}$ by specifying the entry $M_{i,j}$ to be equal to 
    \begin{align*}
         \sqrt{\bP_t[A_t^\star=a_i]\bP_t[A_t^\star=a_j]} \big(\bE_t[ R'_t(a_j)|A^\star_t=a_i]-\bE_t[ R'_t(a_j)]\big)
    \end{align*}
    for all $i,j\in \big[ |\cA| \big]$.
    Then, the expected regret of the TS algorithm is equal to the trace of the matrix $M$. Indeed,
    \begin{align*}
        \bE_t[ &R^\star_t - R_t] \\
        &= \sum_{a\in\cA}\bP_t[A^\star_t=a]\big(\bE_t[ R'_t(a)|A^\star_t=a]-\bE_t[ R'_t(a)]\big) \textnormal{ a.s.}\\
        &=\textnormal{Trace}(M) \textnormal{ a.s.}.
    \end{align*}
    In the same fashion as in~\cite[Proposition 5]{russo2016information}, we relate $\mi_t(\Theta;R_t|\hat{A}_t)$ to the squared Frobenius norm of $M$ as:
    \begin{align*}
        \mi_t(\Theta;&R_t|\hat{A}_t) \\
        &\geq \mi_t(A^\star_t;R_t|\hat{A}_t) \textnormal{ a.s.} \\
        &=\sum_{a_i,a_j \in \cA} \bP_t[A^\star_t = a_i] \bP_t[A^\star_t = a_j]  \\ 
        & \quad \quad  \cdot \KL{\bP_{R'_t(a_j)|\hat{H}^t,X_t,A^\star_t = a_i}}{\bP_{R'_t(a_j)|\hat{H}^t,X_t}} \textnormal{ a.s.}\\
        &\geq \sum_{a_i,a_j \in \cA} \bP_t(A^\star_t = a_i) \bP_t(A^\star_t = a_j) \\
        & \quad \quad \cdot\frac{1}{2\sigma^2}\big(\bE_t[ R'_t(a_j)|A^\star_t=a_i]-\bE_t[ R'_t(a_j)]\big)^2 \textnormal{ a.s.}\\
        &= \frac{1}{2\sigma^2} ||M||_{F}^2 \textnormal{ a.s.},
    \end{align*}
    where the last inequality is obtained again using the Donsker-Varadhan inequality~\cite[Theorem~5.2.1]{gray2011entropy} as in~\cite[Lemma 3]{russo2014learning}.
    Combining the last two equations and using the inequality $\textnormal{trace}(M) \leq \sqrt{\textnormal{rank}(M)} ||M||_F$ ~\cite[Fact~10]{russo2016information}, it comes that
    \begin{align*}
    \Gamma_t = \frac{\bE_t[ R^\star_t - R_t]^2}{\mi_t(\Theta; R_t|\hat{A}_t)}\leq 2 \sigma^2 \frac{\textnormal{Trace}(M)^2}{||M||_{F}^2}\leq 2\sigma^2 \textnormal{Rank}(M) \textnormal{ a.s.}.
    \end{align*}
    The proof concludes showing the rank of the matrix $M$ is upper bounded by $d$. 
    For the sake brevity, we define $\Theta_t \coloneqq \bE_t[\Theta]$ and $\Theta_{t,i} \coloneqq \bE_t[\Theta|A^\star_t=a_i]$ for all $i\in \big[ |\cA| \big]$. We then have $\bE_t[ \innerproduct{\Theta}{m(X_t,a_j)}] = \innerproduct{\Theta_t}{m(X_t,a_j)}$ a.s. and $ \bE_t[ \innerproduct{\Theta}{m(X_t,a_j)}|A^\star_t=a_i]=\innerproduct{\Theta_{t,i}}{m(X_t,a_j)}$ a.s.. Since the inner product is linear, we can rewrite each entry $M_{i,j}$ of the matrix $M$ as
    \begin{align*}
        \sqrt{\bP_t(A_t^\star=a_i) \bP_t(A_t^\star=a_j)}\innerproduct{\Theta_{t,i}-\Theta_t}{m(X_t,a_j)}.
    \end{align*}
    Equivalently, the matrix $M$ can be written as
    \begin{align*}
        \begin{bmatrix}
        \sqrt{\bP_t[A_t^\star=a_1]}(\Theta_{t,1}-\Theta_t)\\
        \vdots\\
        \sqrt{\bP_t[A_t^\star=a_{|\cA|}]}(\Theta_{t,|\cA|}-\Theta_t)
        \end{bmatrix}
        \begin{bmatrix}
        \sqrt{\bP_t[A_t^\star=a_1]}m(X_t,a_1)\\
        \vdots \\
        \sqrt{\bP_t[A_t^\star=a_{|\cA|}]}m(X_t,a_{|\cA|})
        \end{bmatrix}^\intercal.
    \end{align*}
    This rewriting highlights that $M$ can be written as the product of a $|\cA|$ by $d$ matrix and a $d$ by $|\cA|$ matrix and therefore has a rank lower or equal than $\min(d,|\cA|)$.
\end{proof}


\end{document}